\documentclass[10pt,twocolumn,letterpaper]{article}

\usepackage{3dv}
\usepackage{times}
\usepackage{epsfig}
\usepackage{graphicx}
\usepackage{amsmath}
\usepackage{amssymb}

\usepackage{mathrsfs} 
\usepackage{comment}
\usepackage{algpseudocode,algorithm,algorithmicx}

\usepackage{multirow}
\usepackage{amsthm}
\usepackage{makecell}
\newtheorem{theorem}{Proposition}

\def\bfX{{\bf X}} 
\def\bfY{{\bf Y}} 
\def\bfR{{\bf R}} 
\def\bft{{\bf t}} 
\def\bfE{{\bf E}} 
\def\bfy{{\bf y}} 
\def\bfx{{\bf x}} 
\def\bfp{{\bf p}} 
\def\bfT{{\bf T}} 
\def\bfC{{\bf C}} 
\def\bfc{{\bf c}} 
\def\bfF{{\bf F}} 
\def\bff{{\bf f}} 

\def\bfU{{\bf U}} 
\def\bfn{{\bf n}} 
\def\bfI{{\bf I}} 
\def\bfzero{{\bf 0}} 

\def\bfSigma{{\boldsymbol{Y}}} 
\newcommand{\norm}[1]{\left\lVert#1\right\rVert}

\DeclareMathOperator*{\argmin}{arg\,min}
\usepackage{pifont}

\usepackage{color}
\usepackage{enumitem}

\usepackage[pagebackref=true,breaklinks=true,letterpaper=true,colorlinks,bookmarks=false]{hyperref}

\usepackage[labelfont=bf,font={footnotesize}]{caption} 

\definecolor{darkorange}{rgb}{1.0, 0.55, 0.0}

\threedvfinalcopy 

\ifthreedvfinal\fi
\begin{document}

\title{Fast Simultaneous Gravitational Alignment of Multiple Point Sets} 

\author{Vladislav Golyanik \hspace{2.8em}  Soshi Shimada \hspace{2.8em} Christian Theobalt\vspace{10pt}\\ 
Max Planck Institute for Informatics, SIC}

%\pagenumbering{gobble} 
\maketitle 
\begin{abstract}
   The problem of simultaneous rigid alignment of multiple unordered point sets which is unbiased towards any of the inputs has recently attracted increasing interest, and several reliable methods have been newly proposed. 
   While being remarkably robust towards noise and clustered outliers, current approaches require sophisticated initialisation schemes and do not scale well to large point sets.  
   This paper proposes a new resilient technique for simultaneous registration of multiple point sets by interpreting the latter as particle swarms rigidly moving in the mutually induced force fields. 
   Thanks to the improved simulation with altered physical laws and  acceleration of globally multiply-linked point interactions with a $2^D$-tree ($D$ is the space dimensionality), our Multi-Body Gravitational Approach (MBGA) is robust to noise and missing data while supporting 
   more massive 
   point sets than previous methods (with $10^5$ points and more). 
   In various experimental settings, MBGA is shown to outperform several baseline point set alignment approaches in terms of accuracy and runtime. 
   We make our source code available for the community to facilitate the reproducibility of the results\footnote{\url{http://gvv.mpi-inf.mpg.de/projects/MBGA/}}. 
\end{abstract}

\section{Introduction} 
Rigid point set alignment algorithms conventionally handle two point sets,  one of which is a fixed reference, and the other one is a template being  transformed \cite{Besl_McKay_1992}. 
This setting is a well-understood and -studied. 
With more than two point sets, selecting and fixing one point set as a reference will result in a biased solution towards the selected reference. 
A global or groupwise approach handling all point sets on par, \textit{i.e.,} enabling each point set to transform during the alignment \cite{Eggert96, EvangelidisHoraud2018, Lawin2018}, is a recognised alternative to various pairwise schemes requiring a subsequent consensus of transformations \cite{Bergevin1996, HuberHebert01}. 
Groupwise registration arises in different domains of 3D vision and engineering, including 3D reconstruction, CAD modelling, autonomous robot navigation and simultaneous localisation and mapping, among others. 
While remarkable progress has been achieved in groupwise multi-point set alignment, handling of large point sets (with $10^5$ points and more) remains challenging. 
Subsampling of large point sets can result in loss of high-frequency details and low alignment accuracy. 
Hence, our goal is a method which is robust to noise and which can deal with a large number of points with no subsampling. 

\begin{figure}[t!] 
\centering 
\includegraphics[width=1.0\linewidth]{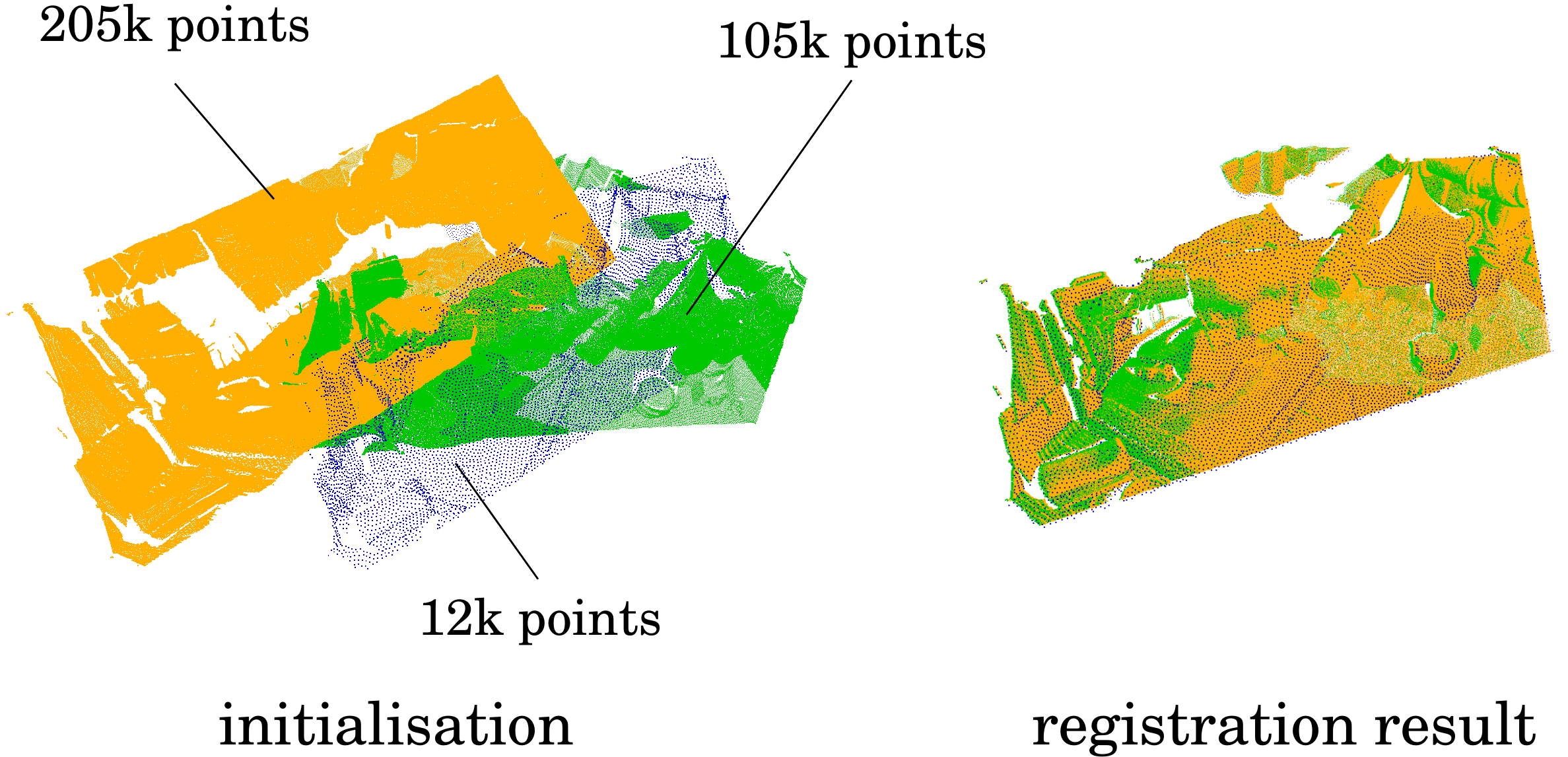} 
\caption{The proposed MBGA simultaneously registers multiple point sets of varying densities into a common reference frame. 
Thanks to the acceleration with a $2^D$-tree, this is the first approach for multi-point-set alignment which performs globally multiply-linked point interactions for large point sets in feasible time (in minutes instead of years). %
} 
\label{fig:multiple_densities} 
\end{figure}

The core variables for deriving the complexity of groupwise point set alignment approaches are the number of input point sets $L$, the average cardinality of the point sets $N$ and a model-specific factor $Q$ which reflects how many operations are expected per point of every point set. %
Thus, we arrive at $\mathcal{O}(Q N L)$ complexity in the general case. 
If point sets are multiply-linked, \textit{i.e.,} every point of one point set interacts with a region of another point set, or nearest-neighbour search is required, then $Q{\sim}N$. 
If data structures for acceleration such as $k$-d trees are used, the worst-case complexity for computing nearest neighbours is still linear (which is more and more probable with the growing point dimensionality  \cite{Friedman1977,Weber1988}) in $N$ implying the overall worst-case complexity to be $\mathcal{O}(Q N^2)$. 
Especially in probabilistic approaches, where all incoming point sets are assumed to be sampled from a single unknown mixture model, %
one can often achieve an accurate alignment with $Q < N$. 
However, there is no guarantee that the same $Q$ will work well for a broader range of inputs.  %
In some cases, setting $Q > N$ is necessary, and 
only $\mathcal{O}(Q N^2)$ can be conservatively assumed without any prior knowledge about the scene. 
\vspace{2mm}
\noindent\textbf{Contributions.} 
Our contributions in this work are: 
\vspace{-1mm}
\begin{itemize}[leftmargin=*,itemsep=1pt]
    \item We present Multi-Body Gravitational Approach (MBGA) --- \textit{the first particle dynamics based method for simultaneous alignment of multiple point sets}, see Fig.~\ref{fig:multiple_densities}. 
    It operates in \textit{the reference-free mode} and models point clouds as rigid particle swarms moving in mutually induced and alternating gravitational force fields (Sec.~\ref{ssec:MBGA}). 
    \item We accelerate the calculation of gravitational potentials acting between particles with a Barnes-Hut $2^D$-tree (quadtree in 2D, octree in 3D) \cite{BarnesHut86} operating on the union of point clouds, and \textit{guarantee the quasi-linear runtime complexity in the total number of points} $N L$ (Secs.~\ref{ssec:point_clustring} and \ref{ssec:computational_complexity}). 
    We show that this data structure can induce a \textit{new shape signature} relying on fitting a polynomial of degree three and considering the cubic surface component at every point (Sec.~\ref{ssec:shape_signature}). 
    The new shape signature helps in handling partial overlaps. 
    \item Experimental evaluation with \textit{state-of-the-art results among non-learning based methods}, in which MBGA outperforms several existing point set alignment approaches in terms of accuracy and runtime. 
    \textit{We empirically confirm the optimality criterion of locally minimal gravitational potential energy (GPE) generalises to more than two point sets handled on par, in a wide range of scenarios.} 
    MBGA aligns point set groups more accurately than 
    the recent pairwise gravitational method BHRGA \cite{BHRGA2019} and several variants of ICP  \cite{Besl_McKay_1992, Fitzgibbon01robustregistration}; the pairwise methods are tested in the \textit{sequential} and \textit{one-to-many} alignment modes. 
    We also align point set with significantly more points that the tested implementation of the groupwise JRMPC method \cite{EvangelidisHoraud2018} can support (Sec.~\ref{sec:experiments}). 
\end{itemize}

MBGA is a globally multiply-linked approach, \textit{i.e.,} \textit{every point of every point set is interacting with every point of every other point set in the system}. 
This is a strong property, considering that existing pairwise methods \textit{de facto} restrict point interactions to local neighbourhoods \cite{TsinKanade2004,  Myronenko2010}. 
Thus, the key strength of MBGA is the unification (for the first time, to the best of our belief) of two seemingly contradicting properties --- the global multiply-linked point interactions and the $\mathcal{O}(L N \log(L N))$ complexity. 

Moreover, MBGA has only a single parameter $\theta$ (the cell opening distance) and handles massive point sets on a mainstream CPU. 
It is easy to implement and contains ${\sim}400$ lines of C++ code. 
Different types of prior information (\textit{e.g.,} point colours, density and prior matches), if available, can be mapped to masses. 
MBGA can be used for completion of room scans obtained with a depth sensor. 
All in all, it offers a valuable alternative to current joint multi-point-set registration methods and extends the scope of available techniques to larger point sets. 

\section{Related Work}\label{sec:related_work} 

Early point set registration algorithms were motivated by emerging 3D scanners that produce partial point clouds that need to be aligned. 
The seminal \textit{iterative closest point} (ICP) algorithm for aligning two point sets~\cite{Besl_McKay_1992,Chen_Medioni_1992}
alternates between transformation estimation \cite{Horn87, Horn88} and local correspondence inference \cite{Elseberg12}. 
The heuristic local correspondence search of ICP makes it prone to erroneous local convergence if badly initialised, 
and sensitive to outliers. 
Different improvements were subsequently proposed for ICP, ranging from accelerating policies for nearest-neighbour search  \cite{GreenspanGodin2001, Nuechter2007} and  relaxation of one-to-one correspondences \cite{Gold97} to more efficient optimisation schemes~\cite{RusinkiewiczLevoy2001, Fitzgibbon2003}. 

The methods \cite{Bergevin1996,Eggert96} are among the first to address the groupwise case. 
Bergevin \textit{et al.}~minimise registration error between all point clouds in a view network using an extension of ICP assuming one point cloud to be fixed. 
Eggert \textit{et al.}~\cite{Eggert96} align point sets into a common, initially unknown, reference frame. 
The poses are updated iteratively by modelling spring forces between matching points, velocities and accelerations by approximate finite element analysis. 
The main drawback of these methods is the expensive correspondence search in every iteration and deterministic correspondence assignment, making them sensitive to noise and large initial misalignments. 
A recent and more robust extension of ICP called Motion Averaged ICP (MAICP) is based on motion averaging on a view graph \cite{GovinduPooja2014}. 

Another class of methods models source and target point clouds as probability density functions~\cite{Chui2000,Myronenko2010,gmmreg}, such as Gaussian Mixture Models (GMMs). 
They can implicitly model multiply-linked assignments and incorporate noise models~\cite{Cao:2018:RHT}. 
The methods probabilistically find a maximum-a-posteriori solution via an expectation-maximisation algorithm or minimise an explicit probability density distance. 
Wider convergence basin was obtained 
by Gaussian mixture decoupling~\cite{Eckart2015} or combining GMM representation and continuous domain mapping with a support vector machine~\cite{Campbell2015}. 
The multiply-linked Kernel Correlation (KC) approach minimises the Renyi's quadratic entropy of the joint system composed of the reference and the transformed template~\cite{TsinKanade2004}. 
In contrast to our method, only local point neighbourhoods are involved in multiply-linked interactions. %

Probabilistic multi-point-set alignment methods lead to remarkable  progress~\cite{EvangelidisHoraud2018, Danelljan2016, Lawin2018}. 
The JRMPC algorithm~\cite{EvangelidisHoraud2018} assumes that a collection of point sets is sampled from the same Gaussian Mixture Model %
and finds the most probable generative GMM \cite{EvangelidisHoraud2018}. 
Lately, Danelljan \textit{et al.}~extended JRMPC by incorporating colour information~\cite{Danelljan2016}, and Lawin \textit{et al.}~improved it for the case of spatially varying point densities~\cite{Lawin2018}. 
JRMPC can become impractical for point sets with $35{-}40k$ points and more. 
In contrast, our method can handle point sets with $10^5$ points and more. 
Deep learning methods for alignment of two point sets \cite{Aoki_2019_CVPR,Wang_2019_ICCV,Choy_2020_CVPR,Pais_2020_CVPR} achieve remarkable accuracy in scenarios with available training datasets, while the generalisability to arbitrary scenarios with different point set characteristics (\textit{e.g.,} varying density or volumetric sampling) is an open question. 
It is ongoing research on how these methods can be generalised to more than two point sets \cite{Gojcic_2020_CVPR}. 
We do not assume available training data, and, thus, can handle arbitrary scenes with moderate overlaps. 
Moreover, MBGA is a correspondence-free approach, \textit{i.e.,} it recovers an alignment in a single joint optimisation and does not require preliminary correspondence extraction. 
Recently, approaches relying on analogies to physical processes gained attention \cite{Deng2014, Golyanik2016, Jauer2019,  golyanik2020quantum}. % 
Deng \textit{et al.}~cast point sets into the Schr\"{o}dinger distance transform representation, and align them by minimising a geodesic distance on a unit Hilbert sphere \cite{Deng2014}. 
Golyanik \textit{et al.}~\cite{Golyanik2016, BHRGA2019} and Jauer \textit{et al.}~\cite{Jauer2019} interpret point sets as rigid swarms of particles with masses moving in gravitational or electromagnetic force fields. 
The optimal alignment corresponds to the state of the minimum potential energy of the system. %
In \cite{golyanik2020quantum}, laws of quantum physics are used to align point sets which are encoded by magnetic fields acting on qubits. 
Our method most closely relates to \cite{Eggert96, Golyanik2016, BHRGA2019} but substantially differs from and improves over them in terms of performance and scalability for multiple point sets. 
Eggert \textit{et al.}~\cite{Eggert96} use Hooke's law and has quadratic worst-case complexity. 
It relies on nearest neighbour correspondence search. 
While the gravitational approach 
\cite{Golyanik2016} is globally-multiply linked, it supports only two point sets with a fixed reference and relies on solving second-order ordinary differential equations which can cause oscillations. 
Besides, these methods include velocity damping for convergence. 
Similarly to \cite{BHRGA2019}, we abstract away from the physically-accurate modelling and alter physical laws for the sake of faster convergence, stability, enhanced numerical properties and fewer parameters which need to be set. 
Both \cite{BHRGA2019} and our MBGA employ a Barnes-Hut tree for the accelerated computation of the potentials. 
We show that MBGA improves the speed and accuracy when handling multiple point sets, over a sequence of point set pair alignments performed by BHRGA \cite{BHRGA2019} and other widely-used methods \cite{Besl_McKay_1992, Fitzgibbon2003}. 
Since MBGA handles all samples on par, its performance is path-invariant and has the loop-closure effect. 

\section{Gravitational Alignment of Two Point Sets}\label{sec:previous_method} 

If a system of two point sets $\{\bfX, \bfY\}$ is given, whereby $\bfX$ is the fixed reference, the input point sets can be aligned by minimising the mutual \textit{gravitational potential energy} (GPE) $\bfU$ of the corresponding system of particles in the force field induced by $\bfX$, parametrised by translation $\bft$ and rotation $\bfR$: %
\begin{equation}\label{eq:main_GPE} 
    \bfU(\bfR, \bft) = -G \sum_{i, j} \frac{m_{\bfy_i} \, m_{\bfx_j}}{\norm{\bfR \, \bfy_i + \bft - \bfx_j}_2 + \epsilon }, 
\end{equation} 
where $m_{\bfy_i}$ and $m_{\bfx_j}$ denote masses, $[\bfy_i] = \bfY$, $[\bfx_j] = \bfX$, $G$ is the  gravitational constant and $\epsilon$ is a softening parameter for preventing near-field singularities. 
In \cite{Golyanik2016}, this energy is minimised implicitly by updating the forces $\vec{\bff_i}$ acting on particles $\bfy_i$, accelerations (using Newton's second law of motion $\vec{\bff_i} = m_{\bfy_i} \vec{a}_i$), velocities $v_i^{t+1}$ and individual point displacements $d_i^{t+1}$: 
\begin{equation}\label{eq:force_explicit} 
    \footnotesize
    \vec{\bff_i} = -G m_{\bfy_i} \sum_j m_{\bfx_j}  \big( \norm{\bfR \bfy_i + \bft - \bfx_j}^2 + \epsilon^2  \big)^{-3/2} \hat{\bfn}_{i,j} - \eta v_i^t, 
\end{equation}
\vspace{-10pt} 
\begin{equation}\label{eq:velocity_displacement_explicit} 
    v_i^{t+1} = v_i + \Delta t \, \frac{\vec{\bff_i}}{m_{\bfy_i}} \;\;\text{and}\;\; d_i^{t+1} = \Delta t \, v_i^{t+1}. 
\end{equation}
In \eqref{eq:force_explicit}, $\eta$ denotes a dissipation constant which determines the portion of the kinetic energy dragged from the system, per template's particle. 
In \eqref{eq:velocity_displacement_explicit}, $\Delta t$ is time or, generally, the forward integration step. 
In every iteration --- once updated --- the unconstrained displacements are added to the current positions. 
A rigid consensus transformation is found using Procrustes analysis \cite{Horn88} relating the previous and current poses. 
The method converges when the difference in GPE of two last successive system states is below some threshold, or the maximum number of iterations is reached. 
\section{Multi-Body Gravitational Approach}\label{ssec:MBGA} 

\subsection{Notations and Assumptions} 

Let $\bfSigma = \{\bfSigma_1, \bfSigma_2, \hdots, \bfSigma_L\}$ be the set of input point sets of dimensionality $D$. 
We call all point sets $\bfSigma_l$, $l \in \{1, \hdots, L \}$ templates, since no reference is selected and all $\bfSigma_l$ are transformed during the alignment. 
The objective of MBGA is to recover a set of rigid transformation parameters $\{\bfR_l, \bft_l\} \in \bfT_l$ aligning all $\bfSigma_l$ into a common reference frame $\mathscr{R}$ (which is not known in advance) and correspondences across all $\bfSigma_l$ for points which have valid matches, see Fig.~\ref{fig:multiple_densities}. 
$\bfR_l$ and $\bft_l$ denote rotation and translation of $\bfSigma_l$,  respectively. %
We denote a rigid transformation operator applying $\bfT_l$ on single or multiple points by $g(\bfT_l, \cdot)$. 
In MBGA, every $\bfSigma_l$ is interpreted as a rigidly transformed particle swarm, and every point induces a gravitational force field and moves in the external (relative to the point set it belongs to) force field. 
We assume that in the general case, point sets are of different cardinalities, \textit{i.e.,} $|\bfSigma_1| \ne |\bfSigma_2| \ne \hdots \ne  |\bfSigma_L|$. 
Particles in our system possess masses, and their movements are modelled following Newton's second law of motion. 
The mass of every point is condensed in an infinitesimal volume of space, and collisions are not possible (no particle merging or splitting are allowed)  \cite{aarseth_2003}. 
To align the point clouds, we need to find a system's configuration corresponding to the minimal GPE. 

\subsection{Our Gravitational Potential Energy Functional}\label{ssec:GPE_FUNCTIONAL} 

It has been shown in \cite{BHRGA2019} that the GPE \eqref{eq:main_GPE} can be improved with the negative elementwise reciprocal transform (NERT), resulting 
after a series of further simplifications in %
\begin{align}\label{eq:BHRAG_FINAL_GPE} 
  \bfE(\bfR, \bft) = \sum_i \sum_j m_{\bfy_i} \, m_{\bfx_j} \norm{\bfR \, \bfy_i + \bft - \bfx_j}_2. 
 \end{align} 
This transform enables a robust optimisation by non-linear least squares (NLLS) with a faster convergence while still preserving the notion of GPE (the higher the distance between the particles, the higher is the GPE). 
NERT makes multiple parameters of GA obsolete, and the remaining ones cover a much broader range of scenarios. 
Furthermore, core acceleration techniques applicable to the classical $\mathtt{n}$-body problem, such as Barnes-Hut octree \cite{BarnesHut86} or fast multipole method \cite{Greengard1987}, are also applicable to \eqref{eq:BHRAG_FINAL_GPE}. 
Thus, we are building upon the currently most accurate method of the gravitational class and propose to minimise the mutual GPE of the system with $L$ point sets $\bfSigma_l$, where the latter are constrained to move rigidly: 
\begin{equation}\label{eq:GPE_main} 
\begin{aligned} 
  \bfE(\bfT) = 
  & \sum_{l = 1}^{L} \,\, \sum_{i = 1}^{|\bfSigma_l|} \, \sum_{\substack{\bfp_j \in \\ \{\bfSigma \setminus \bfSigma_l\} } } 
  \Big( m_{\bfp_i}^l \, m_{\bfp_j} \norm{ g(\bfT_l, \bfp_i^l) - \bfp_j}_2 \Big), 
\end{aligned} 
\end{equation} 
with $\bfT = \{\bfT_l\} = \{\bfT_1, \bfT_2, \hdots, \bfT_L\}$, $\bfp_i^l \in \bfSigma_l$, $\bfp_j \in \{\bfSigma \setminus \bfSigma_l\}$,  rigid transformation operator $g(\cdot,\cdot)$, 
and masses $m_{\bfp_i}^l$ and 
$m_{\bfp_j}$ of $\bfp_i^l$ and $\bfp_j$, respectively. 
There are multiple changes in the GPE for the multi-body case compared to the case with two point sets. 
In \eqref{eq:GPE_main}, every $\bfSigma_l$ is influenced by a transformed  gravitational force field induced by all other point sets excluding  $\bfSigma_l$. 
For each $\bfSigma_l$, there is such $\bfT_l$ so that $\bfE(\bfT)$ is minimal. 
Similar to the two point sets case (see Sec.~\ref{sec:previous_method}), locally-optimal alignment is achieved when the GPE of the system is locally minimal. %

Eq.~\eqref{eq:GPE_main} is minimised iteratively. 
In every iteration, we alternatingly update $\bfT_l$ by changing the roles between the current transformed point set $\bfSigma_l$ and force inducing point sets $\{\bfSigma \setminus \bfSigma_l\}$, \textit{i.e.,} a union of the remaining templates. %
Compared to the pairwise case, there are also substantial differences in the way how the $2^D$-tree is used for acceleration of all particle interactions. 

\subsection{Point Clustering with a $\large{\boldsymbol{2^D}}$-Tree}\label{ssec:point_clustring}

We approximate globally multiply-linked interactions between $\bfSigma_l$ and $\{\bfSigma \setminus \bfSigma_l\}$ by grouping the points of the latter into clusters. 
Instead of exhaustively summing up interactions between every $\bfp_i^l$ and every $\bfp_j$, we cluster $\bfp_j$ --- so that every cluster is represented by its mass located at its centre of mass --- and compute a single interaction between $\bfp_i^l$ and the cluster. 
The further away from $\bfp_i^l$ are $\bfp_j$, the larger clusters can be built and more runtime can be saved, eventually leading to the number of evaluations for every template point $\bfp_i^l$ proportional to $ \log (|\{\bfSigma \setminus \bfSigma_l\}|)$. 

We cluster $\bfp_j$ with a $2^D$-tree, \textit{i.e.,} a quadtree in 2D and an octree in 3D. %
$2^D$-tree is a representation allowing for point agglomeration depending on the desired granularity level \cite{BarnesHut86}. %
Building a $2^D$-tree requires quasi-linear time, and extracting a clustered representation takes logarithmic time in the total number of points. 
The granularity level is determined by the depth at which the clusters are fetched. %
Note that with every level, the cell is partitioned into $2^D$ subcells, and all points falling into a cell at any level belong to that cell. 
The maximum density of the points determines the tree's depth because finally, each point occupies a leaf. %
For more details, the reader can refer to \cite{BarnesHut86} as we use the algorithm exactly as proposed in \cite{BarnesHut86}. 
The difference lies in the points belonging to the $2^D$-tree and how the tree is used. %
We build a $2^D$-tree on the combined set $\bfSigma$. 
To prevent inter-point-set interactions, the masses $m_i^l$ of all points $\bfp_i^l \in \bfSigma_l$ are set to zero, whereas while extracting a clustered representation from the tree for every $\bfp_i^l$, we restore its original non-zero mass. 
Cluster composition and extraction require a single parameter $\theta$. %
Among other things, it determines the amount of the nearest points which are not clustered during GPE calculation. 
Suppose $l$ is the length of the current examined cell, and $\mu$ is the distance from $\bfp_i^l$ to the centre of the cell. 
If $\frac{l}{\mu} < \theta^{-1}$, then the cell's influence on $\bfp_i^l$ is accumulative. 
Otherwise, the subcells are visited recursively until the condition is satisfied or the tree is entirely traversed. 
For every current point set $\bfSigma_l$ with the corresponding $\bfR_l$ and $\bft_l$ being updated, we denote the set of fetched clusters by $\bfC_l$. 
As a set of fetched clusters for every $\bfp_i^l \in \bfSigma_l$ is different, we use the specified notation $\bfC_{l, i}$ for every $\bfp_i^l$ when necessary. 
Thus, the acceleration with a $2^D$-tree alters the GPE \eqref{eq:GPE_main} as follows: 
\begin{equation}\label{eq:GPE_tree} 
\begin{aligned} 
  \bfE_{\bfC}(\bfT) = 
  & \sum_{l = 1}^{L} \,\, \sum_{i = 1}^{|\bfSigma_l|} \, \sum_{\bfc_j \in \bfC_{l,i} } 
  \Big( m_{\bfp_i}^l \, m_{\bfc_j} \norm{ g(\bfT_l, \bfp_i^l) - \bfc_j}_2 \Big), 
\end{aligned} 
\end{equation} 
with the fetched clusters $\bfc_j \in \bfC_{l, i}$ and cluster masses $m_{\bfc_j}$. 

\subsection{Defining Boundary Conditions through Masses}\label{ssec:boundary_conditions} 

In gravitational methods, including MBGA, mass distribution serves as an additional alignment cue. 
We propose a policy for the integration of prior matches into our approach. %
Let $N_c$ be the subset of points for which it is known beforehand that they match among each other across all $L$ inputs. 
By setting point masses in $N_c$ larger than the default masses, the penalty increases for the cases when the points with higher masses do not coincide. 
Thus, we assign to the points from $N_c$ a higher weight which is proportional to the confidence of prior matches (usually, two-three orders of magnitude higher compared to the default masses). 
Such GPE pre-conditioning improves the alignment accuracy under challenging scenarios with noise and missing data in cases when prior matches are available or can be obtained in a pre-processing step. %
Other cues such as colours or geometric descriptors (such as the one proposed in Sec.~\ref{ssec:shape_signature}) can be likewise mapped to masses. 

\subsection{Shape Signature Induced by a  $\large{\boldsymbol{2^D}}$-Tree}\label{ssec:shape_signature} 
Structure variability is the primary registration cue in point set alignment. 
In scenarios with partially-overlapping data and outliers, uniform mass initialisation can result in suboptimal registrations. 
We are encouraged to counterbalance increasing correspondence uncertainties with mass assignments relying on geometric features and propose a new shape signature induced by the surface fitting. %
For every $\bfp_i^l \in \bfSigma_l$, we fit a surface of degree three in its vicinity. 

Suppose $M_i^l$ is a local point neighbourhood of $\bfp_i^l$ and $K = |M_i^l|$, \textit{i.e.,} the number of points in $M_i^l$. 
We find the best approximation of $M_i^l$ by a third degree polynomial of two variables which in general form reads: 
\begin{equation}
\begin{aligned}
    & \;z = a_1 + a_2 x + a_3 y + a_4 x^2 + a_5 x y + \\ & a_6 y^2 + a_7 x^3 + a_8 x^2 y + a_9 x y^2 + a_{10} y^3,   
\end{aligned}
\end{equation}
with the shortcut notation $\bfp_i^l = (x, y, z)$, and the unknown coefficients $\boldsymbol{a} = \{a_1, \hdots, a_{10}\}$. %
Following the standard theory of surface approximation \cite{Nealen2004}, we take partial derivatives of the sum of residuals 
\begin{equation} 
    E = \sum_{i = 1}^{K} \big( r_i^l \big)^2 = \sum_{i = 1}^{K} \big( z_i - 
    (a_1 + a_2 x_i + \hdots + a_{10} y_i^3)  \big)^2 
\end{equation} 
and equate them to zero. As a result, the following linear system is obtained: %
\begin{equation}\label{eq:linear_system_shape_descriptor} 
\hspace{-10pt}
    \begin{bmatrix} 
        1       & x_1       & y_1       & x_1 y_1   & \hdots & y_1^3        \\ 
        1       & x_2       & y_2       & x_2 y_2   & \hdots & y_2^3        \\ 
        \vdots  & \vdots    & \vdots    & \vdots    & \ddots & \vdots       \\ 
        1       & x_K       & y_K       & x_K y_K   & \hdots & y_K^3        \\         
    \end{bmatrix}%
    \begin{bmatrix} 
        a_1 \\ a_2 \\ a_3 \\ \vdots \\ a_{10} 
    \end{bmatrix} = 
    \begin{bmatrix} 
        z_1 \\ z_2 \\ z_3 \\ \vdots \\ z_{K} 
    \end{bmatrix}. 
\end{equation} %
We estimate $\boldsymbol{a}$ in~\eqref{eq:linear_system_shape_descriptor} by solving the corresponding normal equations. 
Finally, we take the sum of the absolute values of $a_7$ and $a_{10}$ coefficients as a shape descriptor, \textit{i.e.,} the indicator of the cubic surface component. 
For highly curved surfaces, the cubic component is strong, and for more flat surfaces it either weak or vanished. 
During the mass initialisation, points with high descriptor values are assigned tenfold masses (this value is scenario-specific), whereas the remaining masses are initialised uniformly with smaller values. 
Our shape descriptor suppresses areas with probable outliers in flat regions as well as noise, and emphasises more descriptive geometric cues reflected in higher masses. 
For every $\bfp_i^l$, computing the shape  descriptor requires finding the local neighbourhood $M_i^l$. 
To find $M_i^l$, we use a $2^D$-tree described in Sec.~\ref{ssec:point_clustring}. 
First, we build a separate $2^D$-tree for the given $\bfSigma_l$. 
Next, we use the property that every point inserted into a $2^D$-tree corresponds to a leaf. 
For a given $\bfp_i^l$, we query the closest points (the leaves) from $\bfC_{l,i}$ by choosing a sufficiently large $\theta > 8.0$. 
Thus, the proposed light-weight shape descriptor executes in $\mathcal{O}(|\bfSigma_l| \, \log(|\bfSigma_l|))$ time. 
Its additional advantages are the continuity, \textit{i.e.,} the property that every point is assigned a shape signature  value unless thresholded (and not only sparse points which is the case with widespread shape descriptors for approximate alignment \cite{Rusu2008, Zhong2009, SipiranBustos2011}), and code reusability (the $2^D$-tree is used unaltered). 

\subsection{Computational and Memory Complexity}\label{ssec:computational_complexity}

\begin{theorem} 
   The computational complexity of MBGA is 
   $\mathcal{O}( L \bar{N} \log (L \bar{N}) )$, 
   where $L$ is the total number of point sets, and $\bar{N}$ is the average number of points in each point set. 
\end{theorem} 
\begin{proof} 
We assume, without loss of generality, that every point set has $\bar{N}$ points on average. 
For $L \bar{N}$ points, a $2^D$-tree is built in $\mathcal{O}(L \bar{N} \log ( L \bar{N} ) )$ time, as each point insertion results in a leaf, and all  point clusters on the path to the leaf need to be updated. 

For computing potentials between the points, each template point requires a different set of clusters in every iteration. 
The complexity of fetching a set of clusters for one point in the two point sets case is $\mathcal{O}(\log \bar{N})$. 
For the two point set case, the clusters for $\bar{N}$ points in the joint $2^D$-tree are fetched in $\mathcal{O}(\bar{N} \log \bar{N})$ time. 
The complexity of fetching clusters for one point set, when the clusters are fetched from $L - 1$ point sets is $\mathcal{O}(\bar{N} \log( (L - 1) \bar{N}))$. 
In every iteration\footnote{MBGA converges within ${\sim}20{-}30$ iterations}, the clusters have to be fetched for every point set out of $L$, and we arrive at $\mathcal{O}( L \bar{N} \log( (L - 1) \bar{N}))$. 
Thus, the total complexity of MBGA is $\mathcal{O}( L \bar{N} \log (L \bar{N})  +  L \bar{N} \log( (L - 1) \bar{N}) )$ which further simplifies to $\mathcal{O}( L \bar{N} \log (L \bar{N}) )$. 
\end{proof}

\begin{theorem} 
The memory complexity of MBGA handling $L \bar{N}$ points is 
$\mathcal{O}(L \bar{N} \log (L \bar{N}))$, with the factor $\log (L \bar{N})$ 
attributable to  
the nodes in the $2^D$-tree \cite{BarnesHut86,BHRGA2019}. 
\end{theorem} 
In practice, we restrict the depth of the $2^D$-tree to $20$. 
In addition to providing the practical upper bound on the memory, this also filters out duplicated particles which would otherwise lead to infinite splits. 
This implies that even for point sets with cardinalities going  beyond tractable in this paper, 
the memory available on modern workstations will be sufficient for the proposed variant of the Barnes-Hut tree. 

\subsection{Gravitational Potential Energy Minimisation}\label{ssec:GPE_minimisation} 

MBGA finds a stable system state in a time-varying force field, \textit{i.e.,} the state with a minimal GPE. 
Every change in the poses of $\bfSigma_l$ causes changes both in the gravitational force field and $\bfC$ (cluster numbers and compositions). %
The alternating role of the transformed point set and force-inducing point set union  makes all $\bfT_l$ mutually dependent. 
Since we optimise $\bfT_l$ alternatingly in every iteration, we need to linearise in the vicinity of the current solution and iterate. 
The right side of \eqref{eq:GPE_tree} is a sum of residuals, and  
$\bfE_{\bfC}(\bfT)$ can be relaxed and minimised by iteratively reweighted non-linear least squares with a Huber-norm \cite{Huber1964}. 
We repeat the following transformation updates until convergence: %
\begin{equation}\label{eq:final_GPE_with_clusters} 
\small 
  \bfT^{t+1} = \bfT^t \oplus \argmin_{\bfT} \sum_{l, i, j} \, w_{i,j,l}(\bfT^{t}) \norm{ g(\bfT_l, \bfp_i^l) - \bfc_j }_\epsilon, 
  \vspace{-8pt}
\end{equation} 
with the weights $w_{i,j,l}(\bfT^{t}) = m_{\bfp_i}^l \, m_{\bfc_j}$, the robustness threshold $\epsilon \in [10^{-4}; 0.1]$ and $\oplus$ denoting the composition of rigid poses. 
After denoting the residuals by $r_{l, i, j}(\bfT) = w_{i,j,l} \norm{ g(\bfT_l, \bfp_i^l) -  \bfc_j}_2$ and collecting them into a single vector-valued function  $\bfF(\bfT) : \mathbb{R}^{6L} \to \mathbb{R}^{ \mathcal{O} ( \bar{N} \log(L \bar{N}) ) } = [r_{1}(\bfT), r_{2}(\bfT), \hdots, r_{\mathcal{O} ( \bar{N} \log(L \bar{N}) )}(\bfT)]^\mathsf{T}$, the overall objective can be compactly written as 
$\bfT = \argmin_{\bfT} \norm{\bfF(\bfT)}_\epsilon$. 
We minimise it with the Levenberg-Marquardt method \cite{Levenberg_44, Marquardt_1963} provided by the \textit{ceres} library \cite{ceres-solver}, with the rotations parametrised by axis-angles. 

\newcommand{\adjustident}[1]{\parbox[t]{\dimexpr\linewidth-\algorithmicindent}{#1}} 
\newcommand{\StateI}[1]{\State \parbox[t]{\dimexpr\linewidth-\algorithmicindent}{#1}} 
\captionsetup[algorithm]{font=normalsize} 

\begin{algorithm}[t]
\small 
\caption{MBGA: Simultaneous Point Set Alignment} 
\label{alg:MBGA} 
\begin{algorithmic}[1] 
\renewcommand{\algorithmicrequire}{\textbf{Input:}} 
\Require $L$ point sets $\bfSigma = \{\bfSigma_1, \bfSigma_2, \hdots, \bfSigma_L \}$ 
\renewcommand{\algorithmicensure}{\textbf{Output:}} 
\Ensure rigid transformations $\bfT_l = \{\bfR_l, \bft_l\}$ aligning all $\bfSigma_l$ into a common unknown reference frame $\mathscr{R}$ 
\State {\bf Initialisation: } $\bfR_l = \bfI$, $\bft_l = \bfzero$, unit masses (\textit{per default}) %
\While{not converged} $\;\;$($t$ is iteration index) 
\StateI{apply current $\bfT_l^t$ to the respective $\bfSigma_l$ (for all $\bfY_l$)} %
\StateI{build a joint $2^D$-tree on the union $g(\bfT_l^t, \bfSigma_l)$ (Sec.~\ref{ssec:point_clustring})} 
\State{\textbf{for} every $\bfSigma_l$  \textbf{do}} 
\StateI{minimise the GPE~\eqref{eq:GPE_tree}: $\bfT^{t+1} = \bfT^t \oplus \argmin_\bfT  \bfE_{\bfC}(\bfT)=$ \\  %
         $\sum_{l = 1}^{L} \,\, \sum_{i = 1}^{|\bfSigma_l|} \, \sum_{\bfc_j \in \bfC_{l,i}  } 	 %
        \Big( m_{\bfp_i}^l \, m_{\bfc_j} \norm{ g(\bfT_l, \bfp_i^l) - \bfc_j}_\epsilon \Big)$ \\
         with the Levenberg-Marquardt method \cite{Levenberg_44, Marquardt_1963, ceres-solver}
        }
\State{\textbf{end for}}
\EndWhile
\end{algorithmic}
\end{algorithm}

\vspace{-12pt}
\paragraph{Solution Initialisation. } 
We initialise $\bfR_l$ as identities and $\bft_l$ as zero vectors. 
Any other values can be used if available from an external process. %
Translation can be approximately resolved in the pre-processing step by bringing all 
point set centroids into coincidence (this is not compulsory). 
Points are initialised with unit masses if a point set is uniformly sampled. 
Otherwise, a variant of a mass normalisation technique or integrating additional alignment cues counterbalancing varying point densities can be applied (see Secs.~\ref{ssec:boundary_conditions}, \ref{ssec:shape_signature}). 
For the convergence reasons, we keep the number of Gauss-Newton solver iterations per every update of the $2^D$-tree low (one-two iterations). 
Otherwise, the gravitational force field will be severely deprecated while point sets are approaching each other. %
MBGA is summarised in Alg.~\ref{alg:MBGA}. 

\section{Experiments}\label{sec:experiments} 

\begin{figure*}[t!] 
\centering 
\includegraphics[width=1.0\linewidth]{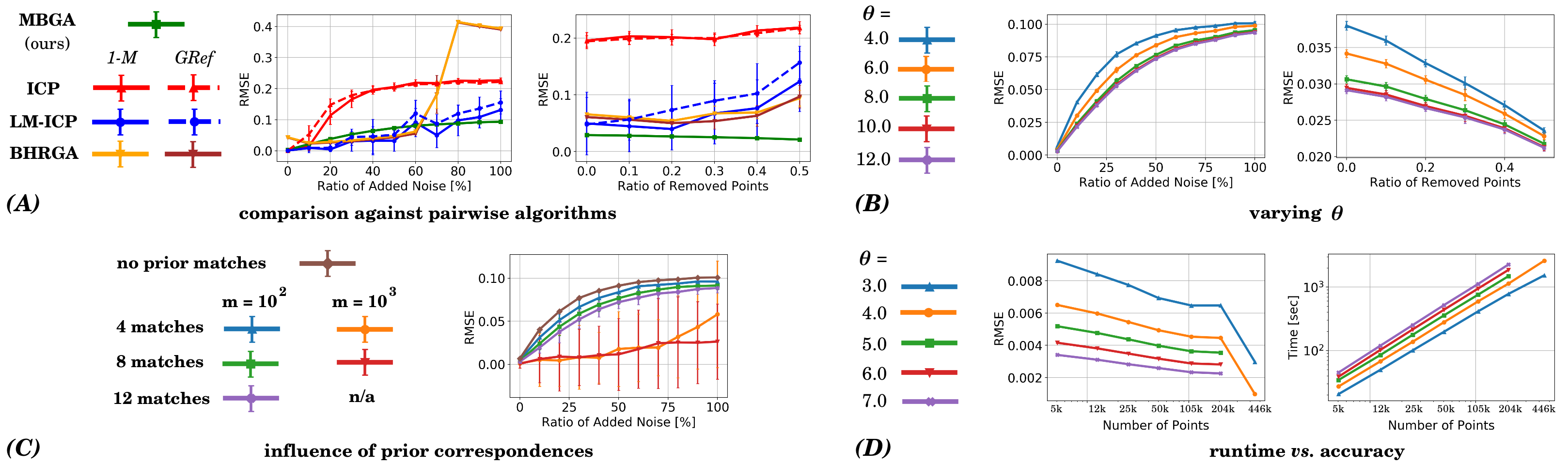}
\caption{ 
The summary of the quantitative results. 
\textbf{\textit{(A): }} Comparison of MBGA to the pairwise point set alignment methods ICP \cite{Besl_McKay_1992}, LM-ICP \cite{Fitzgibbon2003} and BHRGA \cite{BHRGA2019} executed in the modes \textit{one-to-many} (\textit{1-M}) and \textit{growing reference} (\textit{GRef}). 
\textbf{\textit{(B): }} The accuracy of MBGA with different $\theta$. 
\textbf{\textit{(C): }} The accuracy of MBGA with added prior correspondences. 
\textbf{\textit{(D): }} Runtime \textit{vs}~accuracy experiment, \textit{i.e.,} RMSE and the runtime as the functions of $\theta$ and point set cardinalities. 
} 
\label{fig:graphs} 
\end{figure*} 

This section summarises the experimental results. 
MBGA is implemented in C++ and tested on Intel \textit{i7-6700K} (quadcore) under the operating system Ubuntu 16.04. 

\subsection{Quantitative Comparisons}\label{ssec:quantitative_comparisons} 

First, we compare MBGA to the ICP \cite{Besl_McKay_1992} (with no acceleration of correspondence search), LM-ICP \cite{Fitzgibbon2003} and BHRGA \cite{BHRGA2019} approaches for the alignment of two point sets. 
We decimate one frame from the \textit{sleeping 2} sequence of the SINTEL dataset \cite{Butler2012} to $5045$ points and create two rotated and translated copies of it. 
The point sets are composed so that each of the methods can resolve the rotation in the noise-free setting (the maximum transformation angle amounts to $24^\circ$). %
The correspondences across the transformed point sets are known, and we compute the average 3D error $e_{3D}$ for $L$ point sets defined as 
\begin{equation}\label{eq:our_error} 
\small 
  e_{3D} = \begin{pmatrix} L \\ 2\end{pmatrix}^{-1} \sum_{\{i, j\} \in \Phi} \frac{ \norm{ g(\bfT_i, \bfSigma_i) - g(\bfT_{j}, \bfSigma_{j}) }_\mathcal{F}}{ \norm{ g(\bfT_i, \bfSigma_i)  }_\mathcal{F} }, 
\end{equation} 
with $\Phi$ denoting all combinations of two point sets out of $L$, and $\begin{pmatrix} L \\ 2\end{pmatrix} = |\Phi|$ is the total number of combinations. 
ICP, LM-ICP and BHRGA accept two point sets at a time, and we align three point sets either sequentially with the \textit{growing reference} (\textit{GRef}) strategy, or select one reference and align the remaining point sets to it with the \textit{one-to-many} (\textit{1-M}) policy. 
For the pairwise methods, we test all possible combinations of references and templates and report the average $e_{3D}$ and $\sigma$, \textit{i.e.,} the standard deviation of $e_{3D}$, over all runs to avoid the bias towards one of the point sets. 
\textit{1-M} and \textit{GRef} policies require six and twelve alignments for each point cloud triple, respectively. 
We evaluate all methods with the growing level of noise and decreasing degree of partial overlap between the input samples. %
In total, there are eleven different noise levels and six different degrees of partial overlaps (\textit{i.e.,} $17$ tests per method). %
We repeat each experiment for every noise level and degree of partial overlapping ten times. 
Consequently, the pairwise methods require $3060$ alignments each. %
MBGA performs $170$ alignments of point set triples since it does not have to alternate between the choice of the reference. 
\vspace{4pt} 
\noindent\textbf{Uniform Noise.} We add uniformly distributed noise to each point set triple in the bounding sphere ranging in the portion from $5\%$ to $100\%$ of the initial amount of points\footnote{$100\%$ of added noise means that the point set contains $50\%$ of outliers}. 
BHRGA \cite{BHRGA2019} is a state-of-the-art gravitational rigid point set alignment method at the moment of submission. 
Since BHRGA and our MBGA use a $2^D$-tree for the acceleration, we fix $\theta = 12.0$ for both methods in all experiments. 
BHRGA allows for multiple internal iterations per each external NLLS solver iteration, and we set this value to $10$, as recommended in \cite{BHRGA2019} for faster convergence. 

The results of the pairwise methods and our MBGA are shown in Fig.~\ref{fig:graphs}-(A), left. 
The accuracy of MBGA is only slightly affected by noise, and $e_{3D} < 0.1$. 
MBGA outperforms LM-ICP and BHRGA starting from ${\approx}65\%$ of added noise. 
For the lower noise levels, LM-ICP achieves smaller average $e_{3D}$ compared to ICP and MBGA, although it shows the highest $\sigma \leq 0.05$ among all methods. 
At the same time, our approach has higher accuracy compared to ICP already at the noise level of $20\%$. 
Moreover, its $\sigma$ is orders of magnitude lower compared to the competing methods. 
This implies that MBGA is highly robust to noise, \textit{i.e.,} the latter has virtually no effect on its performance. 
This result agrees with previous observations about the strength of the gravitational model \cite{Golyanik2016, BHRGA2019}. 
Additionally, treating all samples on par extra contributes to the accuracy because the final reference frame is obtained automatically. \vspace{4 pt} \\ 
\textbf{Missing Data.} 
Next, we systematically remove points from the same three input point sets and test the methods on their ability to handle missing data. 
The points are removed randomly from each point set, and then the intersection of points with valid correspondences across all point sets is taken to compute $e_{3D}$. 
The portion of the removed points from all three point sets varies in the range $[0.1; 0.5]$. 
On top of that, we add $40\%$ of uniform noise to the samples. 
The results are given in Fig.~\ref{fig:graphs}-(A), right. 
MBGA shows the smallest $e_{3D}$ and $\sigma$ in all experiments. 
Similarly to the case with uniform noise, the accuracy of MBGA is barely affected by the missing data ratio. 
LM-ICP is severely affected by missing data and shows the highest variation in the mean $e_{3D} \in \{0.05; 0.14\}$ among all methods, while BHRGA is the second most accurate method.  
\vspace{4 pt} \\ 
\textbf{Runtime \textit{vs}~Accuracy. } 
We repeat the experiments with noise and missing data with different $\theta$, see  Fig.~\ref{fig:graphs}-(B). 
As expected, larger $\theta$ results in a lower $e_{3D}$. %
At the same time, increasing $\theta$ implies a higher runtime. 
In Fig.\ref{fig:graphs}-(D), we report $e_{3D}$ and the corresponding runtime of MBGA for multiple subsampling factors and $\theta \in [3.0; 7.0]$. 
The dependency of the runtime from point set cardinality is very close to linear,  \textit{i.e.,} we see that with doubling the number of points, the runtime doubles likewise, and this pattern holds for all $\theta$. 
$e_{3D}$ consistently decreases with the increasing number of points in the point sets. 
This confirms that \textit{all points} guide the registration and are significant. 
\vspace{4 pt} \\ 
\textbf{Prior Correspondences.} 
We repeat the test with the uniform noise with prior correspondences.
In total, three modes with four, eight and twelve prior matches each are tested. %
Every prior correspondence in every repetition of the experiment for every noise level is drawn randomly with equal probability from the entire set of points with a valid correspondence (as one point index out of $5045$). 
For the statistical significance, correspondences are drawn randomly ten times for each of ten repetitions of the experiment for each noise level. 
We report $e_{3D}$ and $\sigma$ for each mode with two different masses (weights) of prior matches, \textit{i.e.,} $10^2$ and $10^3$. %
The results and summarised in Fig.~\ref{fig:graphs}-(C). %
With the increasing number of prior matches of the same weight, mean $e_{3D}$ decreases. 
On the other hand, we observe a rise of $\sigma$ with the higher weight. 
This suggests that the choice and distribution of prior correspondences affects $e_{3D}$ stronger with higher weights of prior matches (since prior matches are drawn randomly, they can build a suboptimal configuration when distributed unevenly or degenerately). %
For too large prior weights and a growing number of prior matches, the result becomes strongly influenced by choice of prior matches (which is the case with twelve prior matches and the prior weight of $10^3$). 
\vspace{4 pt} \\ 
\textbf{Comparison to JRMPC.} 
Finally, we align the clean triple of SINTEL frames with JRMPC  \cite{EvangelidisHoraud2018} and MBGA. 
\begin{figure*}[t!] 
\centering 
\includegraphics[width=1.0\linewidth]{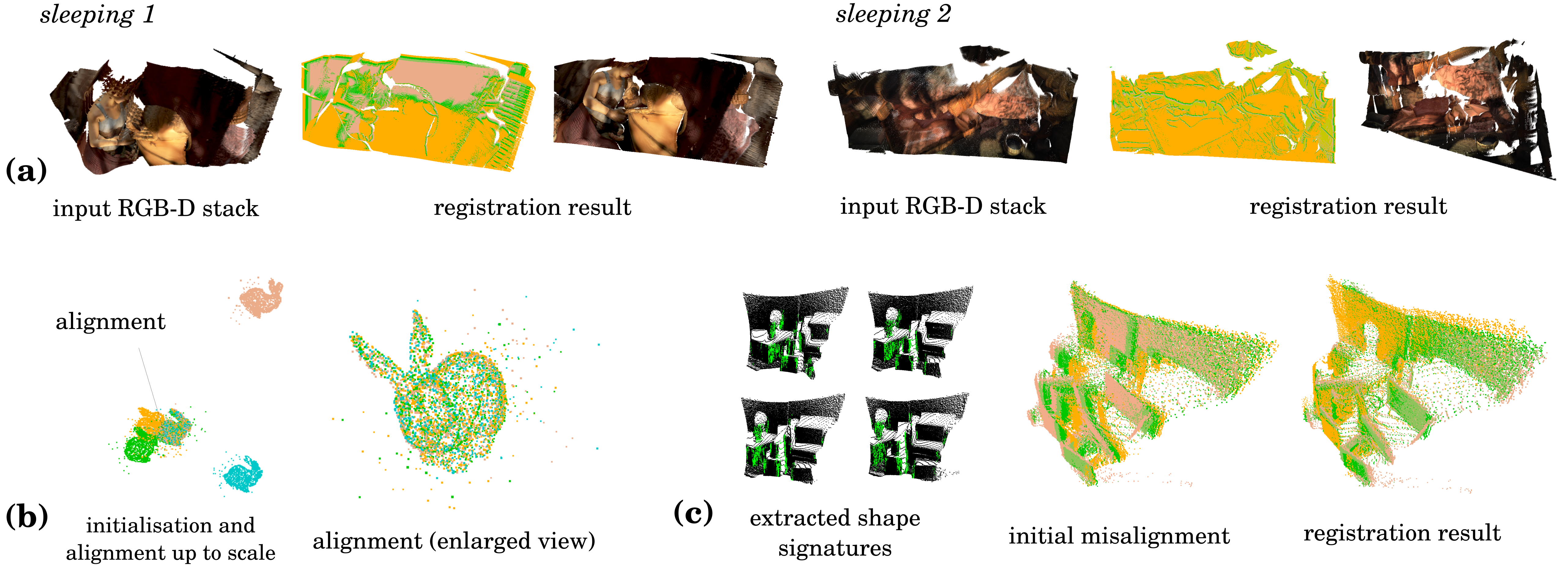}  %
\caption{(a): Qualitative results on the rigid \textit{SINTEL} \cite{Butler2012} scenes, \textit{i.e.,} \textit{sleeping 1} ($4$ frames, taken every $3$ frames) and \textit{sleeping 2} ($5$ frames, taken every other frame). 
Due to the camera movement, either frontal or lateral, the covered areas as well as their sampling densities differ, leading to clustered outliers. 
(b): This experiment with four bunnies \cite{EvangelidisHoraud2018} with different sampling rates and uniform noise demonstrates the robustness of our method to noise and the ability to recover translation without a pre-processing step. 
(c): Qualitative results on the time-of-flight data \cite{EvangelidisHoraud2018}. The shape signatures used for mass initialisation are shown on the left. 
The green colour stands for the emphasised regions, and black marks the shadowed regions with scarce geometric features. 
} 
\label{fig:SINTEL_JOINT} 
\end{figure*} 
\begin{table}[!t] 
\center 
\footnotesize 
    \begin{center} 
        \begin{tabular}{|p{2pt}c||c|c|c||c|c|} \hline 
        & & \makecell{ICP \\\cite{Besl_McKay_1992}} & \makecell{LM-ICP \\\cite{Fitzgibbon2003}} & \makecell{BHRGA\\\cite{BHRGA2019}$^{*}$} & \makecell{JRMPC \\\cite{EvangelidisHoraud2018}$^{*}$} &  \makecell{MBGA \\(ours)$^{*}$} \\ \hline \hline 
        \multirow{2}{*}{\textit{N}} & $e_{3D}$      & $0.2244$  & $0.1435$  &  $0.392$  & \pmb{$1.5E$-$4$} & \pmb{$9.4E$-$2$}      \\ \cline{3-7} 
                                    & $\sigma$      & $6.4E$-$3$ & $3.7E$-$2$ &  $1.4E$-$3$  & \pmb{$5.6E$-$5$} & \pmb{$1.1E$-$3$}  \\ \hline 
        \multirow{2}{*}{\textit{R}} & $e_{3D}$      & $0.2181$  & $0.1403$ & \pmb{$9.6E$-$2$} & \pmb{$1.3E$-$3$} & \pmb{$2.1E$-$2$}       \\ \cline{3-7} 
                                    & $\sigma$      & $8E$-$3$ & $4.1E$-$2$ & \pmb{$1.7E$-$2$} & \pmb{$3.6E$-$4$} & \pmb{$4.8E$-$4$}      \\ \hline 
        \end{tabular} 
    \end{center} 
    \caption{Comparison of tested methods on the data with the highest noise ($100\%$ of added noise, denoted by \textit{"N"}) and missing data ($50\%$ of randomly removed points, denoted by \textit{"R"}) ratios. 
    For the pairwise methods \cite{Besl_McKay_1992, Fitzgibbon2003, BHRGA2019}, the metrics are averaged over the \textit{1-M} and \textit{GRef} policies. "$^{*}$" marks methods achieving $e_{3D} < 0.1$ at least once (highlighted in bold). } 
    \label{table:comparisons} 
\end{table} 
This experiment tests the registration accuracy on clean data. %
For JRMPC, $e_{3D} = 1.58 \cdot 10^{-5}$, and the $e_{3D}$ of MBGA amounts to $0.01$ with $\theta = 8.0$. 
It is known that methods of the gravitational class are not designed for clean data \cite{Golyanik2016}. 
MBGA can repeat the same experiment with the non-subsampled point sets with $464k$ points each, whereas JRMPC,  version~$0.9.4$~\cite{jrmpc-code} (in \textit{Matlab} \cite{MATLAB2016}) cannot launch due to the memory issues and high computational complexity. 

Next, we repeat the experiments with high levels of noise and missing data, %
see 
Table~\ref{table:comparisons}. 
Only BHRGA \cite{BHRGA2019} (missing data), JRMPC \cite{EvangelidisHoraud2018} and MBGA accurately resolve the transformations. 
Even though the proposed MBGA is not the most accurate method compared to probabilistic approaches such as JRMPC, it is the only approach which can execute on large point sets. 
Due to the approximation of potentials with a $2^D$-tree, MBGA achieves a slightly higher $e_{3D}$. 

\subsection{Qualitative Examples} 

Fig.~\ref{fig:SINTEL_JOINT}-(a) shows two registration examples of coloured point cloud batches from the \textit{SINTEL} dataset \cite{Butler2012}. 
We assign masses based on point intensities with a linear mapping. 
MBGA accurately registers batches of four and five RGB-D images. 
Fig.~\ref{fig:SINTEL_JOINT}-(b) illustrates the further experiment with four \textit{Stanford bunnies} \cite{Stanford3D} with added uniform noise followed \cite{EvangelidisHoraud2018}. %
The fusion result for time-of-flight data from \cite{EvangelidisHoraud2018} (four views) is visualised in Fig.~\ref{fig:SINTEL_JOINT}-(c). %
In this experiment, we use the proposed in Sec.~\ref{ssec:shape_signature} shape signature (the emphasised areas are shown on the left). 
The region of interest with a head statue and surfaces are aligned, and the gravitational influence of the wall (a massive clustered outlier with barely any geometric cue) is suppressed. %
\section{Conclusion}\label{sec:conclusion} 

We introduced the first approach for groupwise point set alignment with gravitational particle dynamics which 
alternates between force inducing union of point sets and the transformed point set. 
To the best of our knowledge, it is for the first time that a multiset alignment algorithm operates in the globally multiply-linked manner across all points in all samples. 
At the same time, the quasi-linear time is guaranteed thanks to a $2^D$-tree which is built --- due to the proposed mass shadowing policy --- jointly on all point sets. 
As a result, we can align $5 \cdot 10^5$ points in several minutes. 

MBGA outperforms several widely-used and recent pairwise methods, including BHRGA, when aligning several point clouds with large amounts of noise and missing data. 
Since all inputs are handled on par, our approach is alignment-path-invariant and has a loop-closure effect. 
The experiments have also shown that the number of points matters for the alignment accuracy. 
In the absence of boundary conditions, scenarios with partial overlaps with a predominant common central part work well. 
We further proposed a new shape signature for partially-overlapping data. 
The GPE landscape can be studied analytically in future work to categorise all cases which MBGA can accurately resolve. 
\noindent\textbf{Acknowledgement.} 
This work was supported by the ERC Consolidator Grant 4DReply (770784). 

{\small
\bibliographystyle{ieee}
\bibliography{egbib}

\begin{thebibliography}{10}\itemsep=-1pt

\bibitem{aarseth_2003}
S.~J. Aarseth.
\newblock {\em Gravitational N-Body Simulations: Tools and Algorithms (Chapter:
  The N-body problem)}.
\newblock Cambridge University Press, 2003.

\bibitem{ceres-solver}
S.~Agarwal, K.~Mierle, and Others.
\newblock Ceres solver.
\newblock \url{http://ceres-solver.org}.

\bibitem{Aoki_2019_CVPR}
Y.~Aoki, H.~Goforth, R.~A. Srivatsan, and S.~Lucey.
\newblock Pointnetlk: Robust \& efficient point cloud registration using
  pointnet.
\newblock In {\em Computer Vision and Pattern Recognition (CVPR)}, 2019.

\bibitem{BarnesHut86}
J.~Barnes and P.~Hut.
\newblock A hierarchical {O}(n-log-n) force calculation algorithm.
\newblock {\em Nature}, 324, 1986.

\bibitem{Bergevin1996}
R.~Bergevin, M.~Soucy, H.~Gagnon, and D.~Laurendeau.
\newblock Towards a general multi-view registration technique.
\newblock {\em Transactions on Pattern Analysis and Machine Intelligence
  (TPAMI)}, 18(5):540--547, 1996.

\bibitem{Besl_McKay_1992}
P.~J. Besl and N.~D. McKay.
\newblock A method for registration of 3-d shapes.
\newblock {\em Transactions on Pattern Analysis and Machine Intelligence
  (TPAMI)}, 14(2):239--256, 1992.

\bibitem{Butler2012}
D.~J. Butler, J.~Wulff, G.~B. Stanley, and M.~J. Black.
\newblock A naturalistic open source movie for optical flow evaluation.
\newblock In {\em European Conference on Computer Vision (ECCV)}, pages
  611--625, 2012.

\bibitem{Campbell2015}
D.~Campbell and L.~Petersson.
\newblock An adaptive data representation for robust point-set registration and
  merging.
\newblock In {\em International Conference on Computer Vision (ICCV)}, 2015.

\bibitem{Cao:2018:RHT}
Y.-P. Cao, L.~Kobbelt, and S.-M. Hu.
\newblock Real-time high-accuracy three-dimensional reconstruction with
  consumer rgb-d cameras.
\newblock {\em ACM Trans. Graph.}, 37(5):171:1--171:16, 2018.

\bibitem{Chen_Medioni_1992}
Y.~Chen and G.~Medioni.
\newblock Object modelling by registration of multiple range images.
\newblock {\em Image and Vision Computing - Special issue: range image
  understanding}, 10(3):145--155, 1992.

\bibitem{Choy_2020_CVPR}
C.~Choy, W.~Dong, and V.~Koltun.
\newblock Deep global registration.
\newblock In {\em Computer Vision and Pattern Recognition (CVPR)}, 2020.

\bibitem{Chui2000}
H.~Chui and A.~Rangarajan.
\newblock A feature registration framework using mixture models.
\newblock In {\em Workshop on Mathematical Methods in Biomedical Image Analysis
  (MMBIA)}, pages 190--197, 2000.

\bibitem{Danelljan2016}
M.~Danelljan, G.~Meneghetti, F.~Shahbaz~Khan, and M.~Felsberg.
\newblock A probabilistic framework for color-based point set registration.
\newblock In {\em Computer Vision and Pattern Recognition (CVPR)}, 2016.

\bibitem{Deng2014}
Y.~Deng, A.~Rangarajan, S.~Eisenschenk, and B.~C. Vemuri.
\newblock A riemannian framework for matching point clouds represented by the
  schr\"{o}dinger distance transform.
\newblock In {\em Computer Vision and Pattern Recognition (CVPR)}, pages
  3756--3761, 2014.

\bibitem{Eckart2015}
B.~Eckart, K.~Kim, A.~Troccoli, A.~Kelly, and J.~Kautz.
\newblock Mlmd: Maximum likelihood mixture decoupling for fast and accurate
  point cloud registration.
\newblock In {\em International Conference on 3D Vision (3DV)}, 2015.

\bibitem{Eggert96}
D.~Eggert, A.~W. Fitzgibbon, and R.~B. Fisher.
\newblock Simultaneous registration of multiple range views satisfying global
  consistency constraints for use in reverse engineering.
\newblock In {\em Computer Vision and Image Understanding (CVIU)}, pages
  253--272, 1996.

\bibitem{Elseberg12}
J.~Elseberg, S.~M. Rol, and S.~A. N\"{u}chter.
\newblock A.: Comparison of nearest-neighbor-search strategies and
  implementations for efficient shape registration.
\newblock {\em Journal of Software Engineering for Robotics}, pages 2--12,
  2012.

\bibitem{jrmpc-code}
G.~Evangelidis, D.~Kounades-Bastian, R.~Horaud, and E.~Psarakis.
\newblock Joint registration of multiple point sets, v.~0.9.4.
\newblock \url{https://team.inria.fr/perception/research/jrmpc/}, 2018.
\newblock [accessed on 05.10.2020].

\bibitem{EvangelidisHoraud2018}
G.~D. Evangelidis and R.~Horaud.
\newblock Joint alignment of point sets with batch and
  incrementalexpectation-maximization.
\newblock In {\em Transactions on Pattern Analysis and Machine Intelligence
  (TPAMI)}, 2018.

\bibitem{Fitzgibbon2003}
A.~Fitzgibbon.
\newblock Robust registration of 2d and 3d point sets.
\newblock In {\em British Machine Vision Conference (BMVC)}, pages 1145--1153,
  2003.

\bibitem{Fitzgibbon01robustregistration}
A.~W. Fitzgibbon.
\newblock Robust registration of 2d and 3d point sets, 2001.

\bibitem{Friedman1977}
J.~H. Friedman, J.~L. Bentley, and R.~A. Finkel.
\newblock An algorithm for finding best matches in logarithmic expected time.
\newblock {\em ACM Transactions on Mathematical Software}, 3(3), 1977.

\bibitem{Gojcic_2020_CVPR}
Z.~Gojcic, C.~Zhou, J.~D. Wegner, L.~J. Guibas, and T.~Birdal.
\newblock Learning multiview 3d point cloud registration.
\newblock In {\em Computer Vision and Pattern Recognition (CVPR)}, 2020.

\bibitem{Gold97}
S.~Gold, A.~Rangarajan, C.-P. Lu, and E.~Mjolsness.
\newblock New algorithms for 2d and 3d point matching: Pose estimation and
  correspondence.
\newblock {\em Pattern Recognition}, 31:957--964, 1997.

\bibitem{Golyanik2016}
V.~Golyanik, S.~Aziz~Ali, and D.~Stricker.
\newblock Gravitational approach for point set registration.
\newblock In {\em Computer Vision and Pattern Recognition (CVPR)}, 2016.

\bibitem{golyanik2020quantum}
V.~Golyanik and C.~Theobalt.
\newblock A quantum computational approach to correspondence problems on point
  sets.
\newblock In {\em Computer Vision and Pattern Recognition (CVPR)}, 2020.

\bibitem{BHRGA2019}
V.~Golyanik, C.~Theobalt, and D.~Stricker.
\newblock Accelerated gravitational point set alignment with altered physical
  laws.
\newblock In {\em International Conference on Computer Vision (ICCV)}, 2019.

\bibitem{GovinduPooja2014}
V.~M. Govindu and A.~Pooja.
\newblock On averaging multiview relations for 3d scan registration.
\newblock {\em IEEE Transactions on Image Processing}, 23(3):1289--1302, 2014.

\bibitem{Greengard1987}
L.~Greengard and R.~Vladimir.
\newblock A fast algorithm for particle simulations.
\newblock {\em J. Comput. Phys.}, 1987.

\bibitem{GreenspanGodin2001}
M.~A. Greenspan and G.~Godin.
\newblock A nearest neighbor method for efficient icp.
\newblock In {\em International Conference on 3-D Digital Imaging and Modeling
  (3DIM)}, pages 161--168, 2001.

\bibitem{Horn87}
B.~K.~P. Horn.
\newblock Closed-form solution of absolute orientation using unit quaternions.
\newblock {\em Journal of the Optical Society of America A}, 4(4):629--642,
  1987.

\bibitem{Horn88}
B.~K.~P. Horn, H.~M. Hilden, and S.~Negahdaripour.
\newblock Closed-form solution of absolute orientation using orthonormal
  matrices.
\newblock {\em Journal of the Optical Society of America}, 5(7):1127--1135,
  1988.

\bibitem{HuberHebert01}
D.~F. Huber and M.~Hebert.
\newblock Fully automatic registration of multiple 3d data sets.
\newblock {\em Image and Vision Computing}, 21:637--650, 2001.

\bibitem{Huber1964}
P.~J. Huber.
\newblock Robust estimation of a location parameter.
\newblock {\em Ann. Math. Statist.}, 35(1):73--101, 1964.

\bibitem{Lawin2018}
F.~J\"{a}remo~Lawin, M.~Danelljan, F.~Shahbaz~Khan, P.-E. Forss\'{e}n, and
  M.~Felsberg.
\newblock Density adaptive point set registration.
\newblock In {\em Computer Vision and Pattern Recognition (CVPR)}, 2018.

\bibitem{Jauer2019}
P.~{Jauer}, I.~{Kuhlemann}, R.~{Bruder}, A.~{Schweikard}, and F.~{Ernst}.
\newblock Efficient registration of high-resolution feature enhanced point
  clouds.
\newblock {\em Transactions on Pattern Analysis and Machine Intelligence
  (TPAMI)}, 41(5):1102--1115, 2019.

\bibitem{gmmreg}
B.~Jian and B.~C. Vemuri.
\newblock Robust point set registration using gaussian mixture models.
\newblock {\em Transactions on Pattern Analysis and Machine Intelligence
  (TPAMI)}, 33(8):1633--1645, 2011.

\bibitem{Levenberg_44}
K.~Levenberg.
\newblock {A method for the solution of certain non-linear problems in least
  squares}.
\newblock {\em Quarterly Journal of Applied Mathmatics}, II(2):164--168, 1944.

\bibitem{Marquardt_1963}
D.~W. Marquardt.
\newblock An algorithm for least-squares estimation of nonlinear parameters.
\newblock {\em SIAM Journal on Applied Mathematics}, 11(2):431--441, 1963.

\bibitem{MATLAB2016}
MATLAB.
\newblock {\em Version R2016b}.
\newblock The MathWorks Inc., 2016.

\bibitem{Myronenko2010}
A.~Myronenko and X.~Song.
\newblock Point-set registration: Coherent point drift.
\newblock {\em T-PAMI}, 2010.

\bibitem{Nealen2004}
A.~Nealen.
\newblock An as-short-as-possible introduction to the least squares, weighted
  least squares and moving least squares methods for scattered data
  approximation and interpolation.
\newblock \textit{Report of Discrete Geometric Modeling Group, TU Darmstadt},
  2004.

\bibitem{Nuechter2007}
A.~N{\"u}chter, K.~Lingemann, and J.~Hertzberg.
\newblock Cached k-d tree search for icp algorithms.
\newblock In {\em Sixth International Conference on 3-D Digital Imaging and
  Modeling (3DIM 2007)}, pages 419--426, 2007.

\bibitem{Pais_2020_CVPR}
G.~D. Pais, S.~Ramalingam, V.~M. Govindu, J.~C. Nascimento, R.~Chellappa, and
  P.~Miraldo.
\newblock 3dregnet: A deep neural network for 3d point registration.
\newblock In {\em Computer Vision and Pattern Recognition (CVPR)}, 2020.

\bibitem{RusinkiewiczLevoy2001}
S.~Rusinkiewicz and M.~Levoy.
\newblock Efficient variants of the icp algorithm.
\newblock In {\em Proceedings Third International Conference on 3-D Digital
  Imaging and Modeling}, pages 145--152, 2001.

\bibitem{Rusu2008}
R.~B. Rusu, N.~Blodow, Z.~C. Marton, and M.~Beetz.
\newblock Aligning point cloud views using persistent feature histograms.
\newblock In {\em International Conference on Intelligent Robots and Systems},
  pages 3384--3391, 2008.

\bibitem{SipiranBustos2011}
I.~Sipiran and B.~Bustos.
\newblock Harris 3d: a robust extension of the harris operator for interest
  point detection on 3d meshes.
\newblock {\em The Visual Computer}, 27(11), 2011.

\bibitem{Stanford3D}
{Stanford University Computer Graphics Laboratory}.
\newblock The stanford 3d scanning repository.
\newblock \url{http://graphics.stanford.edu/data/3Dscanrep/}.

\bibitem{TsinKanade2004}
Y.~Tsin and T.~Kanade.
\newblock A correlation-based approach to robust point set registration.
\newblock In {\em European Conference on Computer Vision (ECCV)}, pages
  558--569, 2004.

\bibitem{Wang_2019_ICCV}
Y.~Wang and J.~M. Solomon.
\newblock Deep closest point: Learning representations for point cloud
  registration.
\newblock In {\em International Conference on Computer Vision (ICCV)}, 2019.

\bibitem{Weber1988}
R.~Weber, H.-J. Schek, and S.~Blott.
\newblock A quantitative analysis and performance study for similarity-search
  methods in high-dimensional spaces.
\newblock In {\em International Conference on Very Large Data Bases}, page
  194–205, 1998.

\bibitem{Zhong2009}
Y.~Zhong.
\newblock Intrinsic shape signatures: A shape descriptor for 3d object
  recognition.
\newblock In {\em International Conference on Computer Vision (ICCV)
  Workshops}, pages 689--696, 2009.

\end{thebibliography}
}

\end{document}